\documentclass{article}

\newif\ifarxiv
\arxivtrue

\ifarxiv
	\usepackage[preprint]{neurips_2021}
\else
	\usepackage{neurips_2021}
\fi

\usepackage[utf8]{inputenc} 
\usepackage[T1]{fontenc}    
\usepackage{hyperref}       
\usepackage{url}            
\usepackage{booktabs}       
\usepackage{amsfonts}       
\usepackage{nicefrac}       
\usepackage{microtype}      
\usepackage{xcolor}         
\usepackage{will_style}

\title{Can we globally optimize cross-validation loss? Quasiconvexity in ridge regression}

\author{%
  William T. Stephenson\thanks{Alternate email: wtstephe@gmail.com} \\
  MIT \\
  \texttt{wtstephe@mit.edu} \\
  \And
  Zachary Frangella \\
  Cornell \\
  \texttt{zjf4@cornell.edu}
  \And
  Madeleine Udell \\
  Cornell \\
  udell@cornell.edu
  \And
  Tamara Broderick \\
  MIT \\
  \texttt{tbroderick@mit.edu}
}

\begin{document}

\maketitle

\begin{abstract}
Models like LASSO and ridge regression are extensively used in practice due to their interpretability, ease of use, and strong theoretical guarantees. 
Cross-validation (CV) is widely used for hyperparameter tuning in these models, but do practical optimization methods minimize the true out-of-sample loss? 
A recent line of research promises to show that the optimum of the CV loss matches the optimum of the out-of-sample loss (possibly after simple corrections). 
It remains to show how tractable it is to minimize the CV loss.
In the present paper, we show that, in the case of ridge regression, the CV loss may fail to be quasiconvex and thus may have multiple local optima. 
We can guarantee that the CV loss is quasiconvex in at least one case: when the spectrum of the covariate matrix is nearly flat and the noise in the observed responses is not too high. More generally, we show that quasiconvexity status is independent of many properties of the observed data (response norm, covariate-matrix right singular vectors and singular-value scaling) and has a complex dependence on the few that remain. We empirically confirm our theory using simulated experiments.
\end{abstract}

\section{Introduction}

Linear models, including LASSO and ridge regression,
are widely used for data analysis across diverse applied disciplines.
Linear models are often preferred since they are straightforward to apply in various senses.
In particular, (1) their parameters are readily interpretable.
(2) They have strong theoretical guarantees on quality.
And (3) standard optimization tools are often assumed to find useful parameter and hyperparameter values.
Despite their seeming simplicity, though,
mysteries remain about the quality of inference in linear models.
Consider cross-validation (CV) \citep{stone:1974:earlyCV,allen:1974:PRESS},
the de facto standard for hyperparameter selection across machine learning methods \citep{musgrave:2020:realityCheck}.
CV is an easy-to-evaluate proxy for the true out-of-sample loss.
Is it a good proxy? \citep{homrighausen:2014:l1LOO,homrighausen:2013:l1Kfold,chetverikov:2020:l1Kfold,hastie:2020:l2RidgeCV}
give conditions under which the optimum of the CV loss (possibly with some mild corrections) matches the optimum of the out-of-sample loss in LASSO and ridge regression.
To complete the picture, we must understand whether standard methods
for minimizing the CV loss
find a global minimum.

It would be easy to find a unique minimum of the CV loss if the CV loss were convex.
Alas (though perhaps unsurprisingly),f
we show below that in essentially every case of interest the CV loss is not convex.
Indeed, the usual introductory cartoon of CV loss (left panel of \cref{fig:nonQVXExample}; see also Fig.\ 5.9 of \citet{ESL2} or Fig.\ 1 of \citet{rad:2018:detailedALOO}) is not convex.
But the cartoon CV loss still exhibits a single global minimum and
would be easy to globally minimize with popular approaches like gradient-based methods \citep{do:2007:gradientHyperparams1,maclaurin:2015:gradientHyperparams2,pedregosa:2016:hyperparameterOpt,lorraine:2020:gradientHyperparams3} or grid search \citep{bergstra:2012:random,pedregosa:2011:scikit,hsu:2003:practical}.
Indeed, a more plausible possibility (which holds for the typical cartoon CV loss)
is that the CV loss might be \emph{quasiconvex}.
In one dimension, quasiconvexity implies that any local optimum is a global optimum.

\begin{figure}
  \centering
  \includegraphics[width=0.32\textwidth]{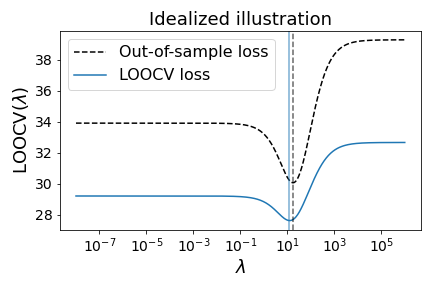}
  \includegraphics[width=0.32\textwidth]{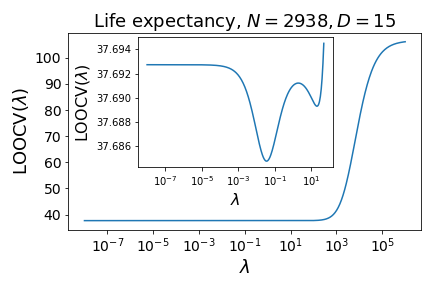} 
  \includegraphics[width=0.32\textwidth]{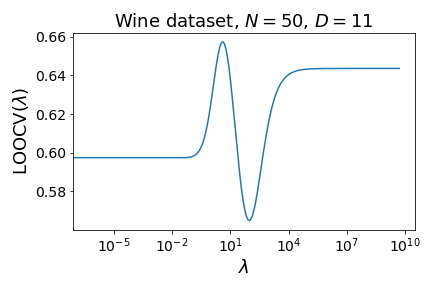} 
  \caption{
  (\emph{Left}): Idealized illustration of the leave-one-out CV loss, $\L$, (blue) and the true out-of-sample loss (black). The minimum of each curve is marked with a vertical line of the corresponding color. (\emph{Center}): CV loss for a life-expectancy prediction problem after some standard data pre-processing (\cref{cond:dataProcessing} of \cref{sec:setup}). (\emph{Right}): CV loss for wine-quality prediction problem on a subset of $N=50$ data points after standard data pre-processing (\cref{cond:dataProcessing} of \cref{sec:setup}).}
  \label{fig:nonQVXExample}
\end{figure}

In what follows, we show that this cartoon need not hold in general.
We consider using the leave-one-out CV (LOOCV) loss, $\L$, to select
the regularization hyperparameter in $\ell_2$-regularized linear regression (i.e.\ ridge regression).
Our first contribution is to demonstrate that this loss can be non-quasiconvex;
see the middle and right panel of \cref{fig:nonQVXExample} for some real-data examples,
which we describe in detail in \cref{sec:qvx}.
Our second contribution is to characterize which aspects of the covariate matrix and observed responses affect quasiconvexity.
We prove that the norm of the responses, the scale of the singular values of the covariate matrix,
and the right singular vectors of the covariate matrix all have no effect on the quasiconvexity of $\L$.
While this result places substantial constraints on what drives the quasiconvexity of $\L$,
we show that the quasiconvexity of $\L$ is unfortunately still a complex function of the remaining quantities.
Our third contribution is to prove conditions under which $\L$ is guaranteed to be quasiconvex.
In particular, we show that if (1) the covariate matrix has a singular value spectrum sufficiently close to uniform,
(2) the least-squares estimator fits the training data sufficiently well,
and (3) the left singular vectors of the covariate matrix are sufficiently regular,
then $\L$ is guaranteed to be quasiconvex.
While the conditions of our theory are deterministic,
we show that they have natural probabilistic interpretations;
as a corollary to our theory, we demonstrate that many of our conditions are satisfied
either empirically or theoretically
by well-specified linear regression problems with i.i.d.\ sub-Gaussian covariates and moderate signal-to-noise ratios.
Through empirical studies, we validate the conclusions of our theory and the necessity of our assumptions.

\section{Setup and notation}
\label{sec:setup}

For $n \in \{1,\dots, N\}$, we observe covariates $x_n \in \R^D$ and responses $y_n \in \R$.
In the present work, we consider the low to modest-dimensional case $D < N$.
We are interested in learning a linear model between the covariates and responses, $\iprod{x_n}{\theta} \approx y_n$, for some parameter $\theta \in \R^D$.
In ridge regression, i.e.\ $\ell_2$-regularized linear regression, we take some $\lambda > 0$ and estimate:
\begin{equation}
	\hat\theta(\lambda) := \argmin_{\theta\in\R^D} \sum_{n=1}^N (\iprod{x_n}{\theta} - y_n)^2 + \frac{\lambda}{2} \n{\theta}_2^2.
	\label{l2Regularization}
\end{equation}
The regularization parameter $\lambda$ is typically chosen by minimizing the cross-validation (CV) loss. Here we study the leave-one-out CV (LOOCV) loss:
\begin{equation}
	\L(\lambda) := \sum_{n=1}^N \left( \iprod{x_n}{\thetan(\lambda)} - y_n \right)^2,
	\label{L}
\end{equation}
where $\thetan(\lambda)$ is the solution to \cref{l2Regularization} with the $n$th datapoint left out.
Let the covariate matrix $X \in \R^{N \times D}$ be the matrix with rows $x_n$, and let the vector $Y \in \R^N$ be the vector with entries $y_n$.
We assume $X$ and $Y$ have undergone standard data pre-processing.
\begin{condition} \label{cond:dataProcessing}
	$Y$ is zero-mean, and $X$ has zero-mean, unit variance columns. Equivalently, where $\1 \in \R^N$ is the vector of all ones, $\1^T Y = 0$ and $X^T \1 = \mathbf{0} \in \R^D$ and for all $d = 1, \dots, D$, $\sum_{n=1}^N x_{nd}^2 = N$.
\end{condition}
Preprocessing $X$ and $Y$ to satisfy  \cref{cond:dataProcessing} represents standard best practice for ridge regression.
First, using an unregularized bias parameter in \cref{l2Regularization} and setting $Y$ to be zero-mean are equivalent;
we choose to make $Y$ zero-mean, as it simplifies our analysis below.
The conditions on the covariate matrix $X$ are important to ensure the use of $\ell_2$-regularization is sensible.
In particular, \cref{l2Regularization} penalizes all coordinates of $\theta$ equally.
If e.g.\  some columns of $X$ are measured in different scales or are centered differently,
this uniform penalty will be inappropriate.

\section{LOOCV loss is typically not convex and need not be quasiconvex}
\label{sec:qvx}

If the LOOCV loss $\L$ were convex or quasiconvex in $\lambda$, then any local minimum of $\L$ would be a global minimum, and we could trust
 gradient-based optimization methods or grid search methods to return a value near a global minimum. We next see that unfortunately $\L$ is typically not convex and is often not even quasiconvex.
First we show that, in essentially all cases of interest, $\L$ is \emph{not} convex.
\begin{prop} \label{prop:LnotConvex}
	If $\lambda = \infty$ is not a minimum of $\L$, then $\L$ is not a convex function.
\end{prop}
\begin{proof}
	For the sake of contradiction, assume $\L$ is convex and $\lambda = \infty$ is not a minimum of $\L$.
	This implies that there is some maximal $\lambda^* < \infty$ such that $\L'(\lambda^*) = 0$.
	Let $\delta := \L'(\lambda^* + 1)$.
	By convexity, $\L'' \geq 0$, so we know that $\delta > 0$ and that for $\lambda \geq \lambda^* + 1$, we have $\L'(\lambda) \geq \delta$.
	Thus for $\lambda \geq \lambda^* + 1$, we have $\L(\lambda) \geq \lambda \delta$.
	So $\lim_{\lambda \to \infty} \L(\lambda) = \infty$.
	However, inspection of $\L$ shows $\lim_{\lambda \to \infty} \L(\lambda) = \sum_{n=1}^N y_n^2 < \infty$, which is a contradiction.
\end{proof}
We say that the result covers essentially all cases of interest:
if $\L$ continues to decrease as $\lambda \rightarrow \infty$,
then there is so little signal in the data that the zero model $\theta = \mathbf{0} \in \R^D$
is the optimal predictor according to LOOCV.

Although $\L$ is generally not convex,
$\L(\lambda)$ might still be easy to optimize if it satisfies an appropriate
generalized notion of convexity.
 To that end, we recall the definition of quasiconvexity.
\begin{defn} \label{defn:quasiconvexity}
	A function $f: \R^p \to \R$ is \emph{quasiconvex} if its level sets are convex.
\end{defn}
In one-dimension (i.e.\ $p=1$ in \cref{defn:quasiconvexity}), quasiconvexity guarantees that any local optimum is a global optimum, just as convexity does.
This property is often the key consideration in practical optimization algorithms.
Moreover, it is not unreasonable to hope that the CV loss is quasiconvex:
typical illustrations of the CV loss are not convex but are
quasiconvex; see e.g.\ \citet[Fig.\ 5.9]{hastie:2015:sls}, \citet[Fig.\ 1]{rad:2018:detailedALOO}, or the left panel of \cref{fig:nonQVXExample}.
Illustrations of the out-of-sample loss are also typically quasiconvex; see e.g.\ Fig.\ 3.6 of \citet{PRML}.

Unfortunately, we next demonstrate that the CV loss derived from real data analysis problems can be non-quasiconvex.
Our first dataset contains $N = 2{,}938$ observations of life expectancy,
along with $D = 20$ covariates such as country of origin or alcohol use;
see \cref{app:realData} for a full description.
In this case, after pre-processing according to \cref{cond:dataProcessing}, $\L$ for the full dataset is quasiconvex.
But now consider some additional standard data pre-processing.
Practitioners often perform principal component regression (PCR) with the aim of reducing noise in the estimated $\theta$.
That is, they take the singular value decomposition of $X = USV$
and produce an $N \times R$ dimensional covariate matrix $X'$ by retaining just
the top $R$ singular values of $X$: $X' = U_{\cdot,:R} S_{:R}$.
With this pre-processing, the resulting LOOCV curve $\L$ is non-quasiconvex for many values of $R$;
in the center panel of \cref{fig:nonQVXExample} we show one example for $R = 15$.

Our second dataset consists of recorded wine quality of $N = 1{,}599$ red wines.
The goal is to predict wine quality from $D = 11$ observed covariates relating to the chemical properties of each wine; see \cref{app:realData} for a full description.
We find that subsets of this dataset often exhibit non-quasiconvex $\L$.
In the right panel of \cref{fig:nonQVXExample}, we show $\L$ for a subset of size $N = 50$.
We see that this plot contains at least two local optima, with substantially different values of $\lambda$ and substantially different values of the loss.
A gradient-based algorithm initialized sufficiently far to the left would not find the global optimum,
and grid search without sufficiently large values would not find the global optimum.

Now we know that $\L$ can be non-quasiconvex for real data.
Given the difficulty of optimizing a function with several local minima,
we next seek to understand \emph{when} $\L$ is quasiconvex or not.

\section{What does the quasiconvexity of $\L$ depend on?}
\label{sec:qvxDependence}

We have seen that $\L$ can be quasiconvex or non-quasiconvex,
depending on the data at hand. If we could determine the quasiconvexity of $\L$ from the data,
we might better understand how to efficiently tune hyperparameters from the CV loss.
In what follows, we start by showing that the quasiconvexity of $\L$ is, in fact,
independent of many aspects of the data (\cref{prop:whatMatters}).
We will see, however, that the dependence of quasiconvexity on the remaining aspects (though they are few) is complex.

A linear regression problem has a number of moving parts.
The response $Y$ may be an arbitrary vector in $\R^N$, and
the covariate matrix $X$ can be written in terms of its singular values and left and right singular vectors.
More precisely, let $X = USV^T$ be the singular value decomposition of the covariate matrix $X$,
where $U \in \R^{N \times D}$ is an $N \times D$ matrix with orthonormal columns,
$S$ is a diagonal matrix with non-negative diagonal,
and $V \in \R^{D \times D}$ is an orthonormal matrix.
With this notation in hand, we can identify aspects of the problem that
do not contribute to quasiconvexity in the following result, which is proved in \cref{app:whatMatters}.
\begin{prop} \label{prop:whatMatters}

	The quasiconvexity of $\L$ is independent of
	\begin{enumerate}
		\item the matrix of right singular vectors, $V$,
		\item the norm of the responses, $\n{Y}_2$, and
		\item the scaling of the singular values (i.e.\ changing $S$ into $S / c$ for $c \in \R_{> 0}$),
	\end{enumerate}
	in the sense that altering any of these quantities does not change whether or not $\L$ is quasiconvex.
\end{prop}
\begin{rem} \label{rem:reduction}
Assume \cref{cond:dataProcessing} holds.
Then by \cref{prop:whatMatters}, without loss of generality we may (and do) assume
that $V = I_D$ and that $Y$ is a vector on the unit $(N-2)$-sphere.
\end{rem}
Recall $X$ has zero-mean columns by preprocessing the data (\cref{cond:dataProcessing}).
By \cref{prop:whatMatters}, we assume without loss of generality $V = I_D$.
Thus, the columns of $X$ have zero mean when $U^T \mathbf{1} = 0$,
where $\mathbf{1} \in \R^N$ is the vector of all ones.
Hence the quasiconvexity of $\L$ depends on three quantities:
(1) the matrix of left singular vectors, $U$, an orthonormal matrix with $\mathbf{1} \in \R^N$ in its left null-space, (2) the (normalized) vector $Y$ which sits on the unit $(N-2)$-sphere, and (3) the (normalized) singular values.

Now that we know the quasiconvexity of $\L$ depends on only three quantities,
we might hope that quasiconvexity would be a simple function of the three.
To investigate this dependence, we consider the case of $N = 3$ and $D = 2$,
since this case is particularly easy to visualize.
To see why it is easy to visualize, first note that $Y$ is a three-dimensional vector.
But due to the constraints $\| Y\|_2 = 1$ and $\mathbf{1}^T Y = 0$,
$Y$ can be taken to sit on the unit circle and hence can be parametrized by a scalar.
Second, note that the matrix of left singular vectors $U$ is parameterized by two orthonormal vectors,
$U_{\cdot1}$ and $U_{\cdot2}$, each on the unit 2-sphere.
As both vectors must be orthogonal to $\1 \in \R^3$,
we can parameterize $U_{\cdot 1}$ and $U_{\cdot 2}$ by two orthonormal vectors on the unit circle.
We parametrize $U_{\cdot 1}$ by a scalar that determines $U_{\cdot 2}$
up to a rotation of $U_{\cdot 2}$ by $\pi$ .
We fix a rotation for $U_{\cdot 2}$ relative to $U_{\cdot 1}$ and singular values $S$,
so that the quasiconvexity of $\L$ is parameterized by two scalars.

\begin{figure}
\centering
	\begin{tabular}{ccc}
		\includegraphics[scale=0.27]{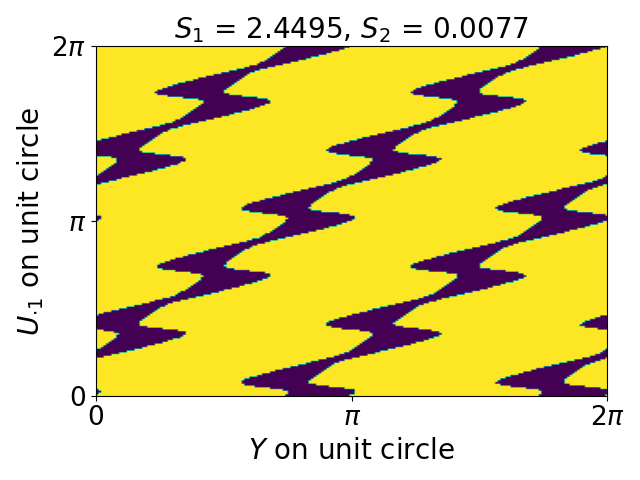} &
		\includegraphics[scale=0.27]{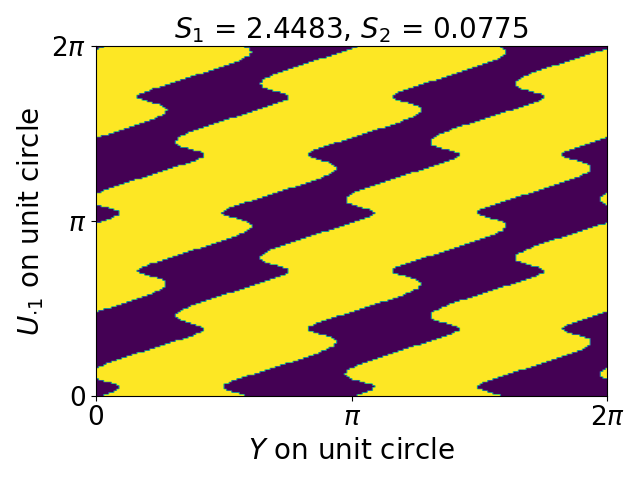} &
		\includegraphics[scale=0.27]{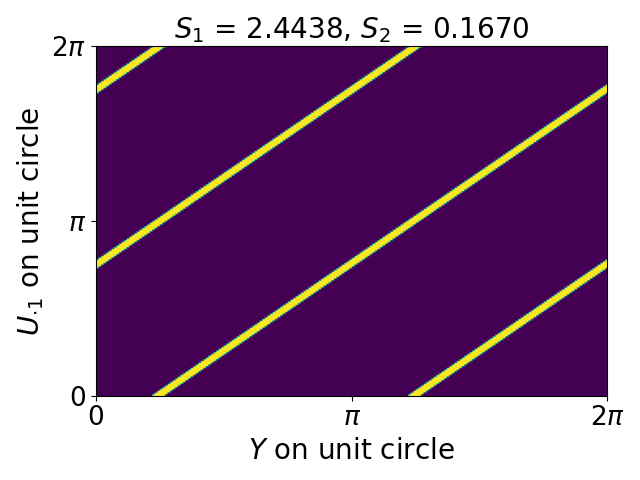}
	\end{tabular}
\caption{How does quasiconvexity depend on $Y$ and $U$, for data with $N=3$ and $D=2$?
The left, center, and right panels each correspond to a different setting of the singular values $S$.
We divide the unit circle for each of $U$ and $Y$ into 100 equally spaced points.
We check if $\L$ is quasiconvex over this $200 \times 200$ grid.
If $\L$ is quasiconvex at a point, that point is colored dark purple;
if $\L$ is not quasiconvex, the point is yellow.}
\label{fig:qvxComplicated}
\end{figure}

\cref{fig:qvxComplicated} is the resulting visualization. Precisely, to make \cref{fig:qvxComplicated}, we fix an orientation between $U_{\cdot 1}$ and $U_{\cdot 2}$: here, a rotation of $\pi / 2$ on the unit circle. We create a uniform grid over the unit circles for $U_{\cdot 1}$ and $Y$. \cref{fig:qvxComplicated} visualizes the quasiconvexity of $\L$ as we move over this grid for three different settings of the singular values, corresponding to the left, center, and right panels.
We see that, in general, $U$, $Y$, and $S$ have a complicated interaction to determine the quasiconvexity of $\L$. The values of $Y$ that produce quasiconvexity depend on $U$.
No setting of $Y$ or $U$ guarantees quasiconvexity.

One trend does seem clear from \cref{fig:qvxComplicated}: as the singular values become more similar (moving from left to right in the panels of \cref{fig:qvxComplicated}), the fraction of $Y$ values and $U$ values that correspond to non-quasiconvexity (yellow regions) shrinks.
Based on this behavior, one might conjecture that a sufficiently uniform spectrum of the covariate matrix could guarantee the quasiconvexity of $\L$.

\section{Quasiconvexity of $\L$ with a nearly uniform spectrum $S$} \label{sec:uniformSpectrum}

We now build on the conjecture of the previous section to show that we can, in fact, guarantee quasiconvexity in certain cases.
In particular, we will show conditions under which a sufficiently uniform spectrum $S$ of the covariate matrix $X$
guarantees that $\L$ is quasiconvex.

One might hope that for large $N$, eventually any $Y$ or $U$ would yield quasiconvexity.
Even when $S$ is exactly uniform, our experiments in \cref{sec:uniformSpectrumExperiments}
show that we cannot expect such a statement.
Rather, we will devise conditions on $Y$ and $U$ to avoid ``extreme'' settings for either quantity.
With these conditions, our main theorem will show that a sufficiently flat spectrum $S$ does indeed guarantee quasiconvexity of $\L$.
When our theorem applies, we can safely terminate an optimization procedure
at the first local minimum of $\L$ that we discover.

We first establish notation and then state our assumptions.
\begin{defn}
	Let $\hat\theta := \big(X^T X\big)\inv X^T Y$ be the least-squares estimate.
	Define the least-squares residuals $\Ehat := Y - X\hat\theta$.
	Let $\ehat$ be the $n$th entry of $\Ehat$.
\end{defn}
Note that $\hat\theta$ is well-defined since we have assumed that $D < N$ and that $X$ is full-rank.
For the tractability of our theory, all of our assumptions and conclusions will be asymptotic in $N$; in particular, our assumptions will use big-O and little-o statements, which are to be taken with respect to $N$ growing large.
Since LOOCV is useful precisely for finite $N$, we are careful to show in our experiments (\cref{sec:uniformSpectrumExperiments}) that these asymptotics take hold for small $N$.

Our first assumption concerns the magnitude of $\Ehat$.
\begin{assumption}
	$(1/N)\sum_{n=1}^N \ehat^2$ is $O(1)$ (i.e.\ it does not grow with $N$).
	\label{assum:ONehat}
\end{assumption}
This assumption is fairly lax.
For example, suppose our linear model is well-specified. In particular, suppose there exists some $\theta^* \in \R^D$ such that $y_n = \iprod{x_n}{\theta^*} + \eps_n$ where the $\eps_n$ are i.i.d.\ $N(0,\sigma^2)$ for some $\sigma > 0$.
Stack the $\eps_n$ into a vector $E \in \R^N$. Then $\|\hat E \|^2 = \| (I_N - U U^T) E \|^2 < \| E \|^2$. Since $(1/N)\| E \|^2$ is $O(1)$ almost surely, it follows that \cref{assum:ONehat} holds almost surely in this well-specified linear model.
We emphasize that \cref{assum:ONehat} depends on the residuals of the least squares estimate,
not (directly) on the noise in the observations.

Our next assumption governs the size of the least squares estimate $\hat\theta$.
\begin{assumption}
	$\| \hat\theta \|$ is $O(1)$ (i.e.\ it does not grow with $N$).
	\label{assum:orderThetaHat}
\end{assumption}
Again, this is a lax assumption.
For example, given any statistical model for which $\hat\theta$ is a consistent estimator for some quantity, \cref{assum:orderThetaHat} holds.

Our next assumption constrains the uniformity of the left-singular value matrix $U$ with rows $u_n \in \R^D$.
\begin{assumption}
	We have $\max_n \nun := \numax = O(N^{-p})$ for some $p > 1/2$.
	\label{assum:ordernu}
\end{assumption}
\cref{assum:ordernu} is an assumption about the coherence of the $U$ matrix, a quantity of importance in compressed sensing and matrix completion \citep{candes:2009:matrixCompletion}. In particular,
\cref{assum:ordernu} requires that the coherence of $U$ decay sufficiently fast as a function of $N$.
Suppose we remove the condition that $U$ have zero-mean columns (see \cref{cond:dataProcessing} and the discussion after \cref{rem:reduction}) and assume a uniform distribution over valid $U$ (i.e., matrices with orthonormal columns); then \cref{assum:ordernu} is known to hold with high probability for any $p \in (1/2,1)$ \citep[Lemma 2.2]{candes:2009:matrixCompletion}.

There do exist matrices $U$ with orthonormal zero-mean columns that do not satisfy \cref{assum:ordernu}.
For instance, take some small $N_0$ (say $N_0 = 5$) and a valid $U'$ for this $N_0$. Then, for $N > N_0$, form $U$ by appending $\mathbf{0} \in \R^{(N-N_0)\times D}$ to the bottom of $U'$.
This construction yields an $N \times D$ matrix $U$ with orthonormal and zero-mean columns
for which the largest $\nun$ (across $n \in \{1,\ldots,N\}$) is constant as $N$ grows.
Still, in our experiments in \cref{sec:uniformSpectrumExperiments,app:orderNu}, we confirm that, for a uniform distribution over orthonormal $U$ with zero-mean columns, \cref{assum:ordernu} holds with high probability.

Our final assumption is a technical assumption relating $\nun$, $\ehat$, and $\hat\theta$.
\begin{assumption}
	The following quantity is positive and $\Theta(1)$ (i.e.\ is bounded away from zero and does not grow with $N$):
	$
		\| \hat\theta \|^2  - \sum_{n=1}^N \nun \left( \uth^2 + 2\ehat^2 \right)
	$
	\label{assum:positiveCoeffs}
\end{assumption}
Roughly, this assumptions means that the largest $\nun$ and $\ehat^2$ values do not occur for the same values of $n$.
To see this relation, note that \cref{assum:ordernu} implies $\| \hat\theta\|^2 -  \sum_n \nun \uth^2 \geq (1-O(N^{-p}) )\| \hat\theta\|^2$.
If we assume that $\| \hat\theta \|^2 = \Theta(1)$ (i.e.\ \cref{assum:orderThetaHat} holds and $\hat\theta$ does not converge to $\mathbf{0} \in \R^D$), then we find
$
	\|\hat\theta\|^2 - \sum_n \nun \uth^2 = \Theta(1).
$
So, we need only that $\sum_n \nun \ehat^2 = o(1)$ for \cref{assum:positiveCoeffs} to hold;
e.g.\ we need that the largest values of $\nun$ and the largest values of $\ehat^2$ typically do not occur for the same values of $n$.

With our assumptions in hand, we can now state our main theorem. Our theorem relates the uniformity of the singular values of $X$ to the quasiconvexity of $\L$.
As we have shown in \cref{prop:whatMatters}, the scaling of the singular values does not matter for the quasiconvexity of $\L$. We therefore take the singular values to be nearly uniform around $\1 \in \R^D$.
\begin{thm} \label{thm:uniformSpectrum}
	Take \cref{assum:ONehat,assum:orderThetaHat,assum:ordernu,assum:positiveCoeffs}.
	For $N$ sufficiently large, there is a neighborhood $\Delta$ of $\mathbf{1} \in \R^D$ such that if the spectrum $S \in \Delta$, $\L(\lambda)$ is quasiconvex.
\end{thm}
\begin{proof}[Proof sketch:]
For one-dimensional functions $\L$, a sufficient condition for quasiconvexity is that for all $\lambda$ such that $\L'(\lambda) = 0$, we have $\L''(\lambda) > 0$ \citep[Chapter~3.4]{boyd:2009:cvxOpt}. We first bound the region in which $\L'$ can be zero and then show $\L''$ is positive over this region. See \cref{app:uniformSpectrum} for a full proof.
\end{proof}

In \cref{sec:qvxDependence}, we showed it can be difficult to guess when $\L$ is quasiconvex.
But \cref{thm:uniformSpectrum} yields one condition that guarantees $\L$ is quasiconvex:
when $X$ has a nearly uniform spectrum.
A natural question then is: when is the spectrum of $X$ nearly uniform?
As it happens, a uniform spectrum occurs under standard assumptions,
for example,
when the $x_{nd}$ are i.i.d.\ sub-Gaussian random variables.
\begin{defn}[e.g.\ \citep{vershynin:2017:hdpBook}]
	A random variable $Q$ is sub-Gaussian if there exists a constant $c > 0$ such that $\mathbb{E}[\exp\big( Q^2 / c^2 \big) ]\leq 2$.
\end{defn}
\begin{cor} \label{cor:subGaussian}
	Take \cref{assum:ordernu,assum:positiveCoeffs}. Assume we have a well-specified linear model for some $\theta^* \in \R^D$; in particular, assume $y_n = \iprod{x_n}{\theta^*} + \eps_n$, where $\eps_n \overset{i.i.d.}{\sim} \mathcal{N}(0,\sigma^2)$.
	If $\sigma$ is sufficiently small, $\hat\theta$ is consistent for $\theta^*$, and the entries of the covariate matrix $x_{nd}$ are i.i.d.\ sub-Gaussian random variables, then $\L$ is quasiconvex with probability tending to 1 as $N \to \infty$.
\end{cor}
\begin{proof}[Proof sketch:]
\cref{assum:ONehat,assum:orderThetaHat} hold for a well-specified linear model.
If the entries of $X$ are i.i.d.\ sub-Gaussian random variables, standard concentration inequalities imply that its spectrum is nearly uniform with high probability; hence the result of \cref{thm:uniformSpectrum} applies.
See \cref{app:subGaussian} for a full proof.
\end{proof}

\section{\cref{thm:uniformSpectrum} in practice}
\label{sec:uniformSpectrumExperiments}

\begin{figure}
	\centering
	\begin{tabular}{cc}
		\includegraphics[scale=0.43]{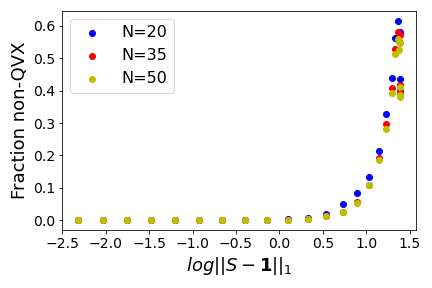} &
		\includegraphics[scale=0.43]{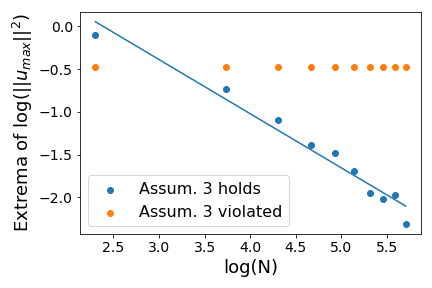}
	\end{tabular}
	\centering
	\includegraphics[scale=0.43]{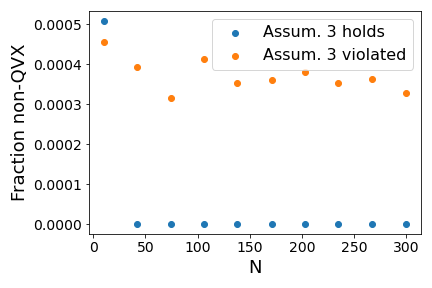}
	\caption{(\emph{Upper left}):
	We generate many datasets and plot the fraction that are not quasiconvex,
	varying $N$ and the distance of the spectrum from uniformity ($\|S - \mathbf{1}\|_1$).
(\emph{Upper right}): We generate two sets (orange, blue) of left-singular vector matrices $U$. In the blue case, we check that the maximum of $\log \numax$ across all $U$ for a particular $N$ decreases roughly linearly on a log-log plot (i.e.\ the blue set satisfies \cref{assum:ordernu}). In the orange case, we check that the minimum of $\log \numax$ across all $U$ for a particular $N$ is roughly constant (i.e.\ the orange set does not satisfy \cref{assum:ordernu}). (\emph{Lower}): For all the $U$ matrices from the upper right plot,
we generate many datasets and plot the fraction that are not quasiconvex.}
	\label{fig:deltaSizeAndNuAssum}
\end{figure}

In \cref{sec:uniformSpectrum}, we established a number of assumptions that we then required in \cref{thm:uniformSpectrum} to prove that $\L$ is quasiconvex. Moreover, \cref{thm:uniformSpectrum} says that this quasiconvexity holds for some unspecified neighborhood $\Delta$ of the uniform spectrum $\mathbf{1} \in \R^D$ and asymptotically in $N$. In this section, we empirically investigate the necessity of our assumptions. And we check how large the neighborhood $\Delta$ in \cref{thm:uniformSpectrum} might be and how quickly quasiconvexity appears as a function of $N$. To that end, we here focus on the small to modest-$N$ regime of $N < 100$.

The only software dependency for our experiments is \texttt{NumPy} \citep{numpy}, which uses the BSD 3-Clause ``New'' or ``Revised'' License.

\textbf{How large is the neighborhood $\Delta$ in \cref{thm:uniformSpectrum}?}
We first show that the neighborhood $\Delta$ is substantial, even for small to moderate $N$.
We fix $D = 5$.
To generate various spectra of $X$, we set $S_d = e^{\alpha d} / e^{\alpha D}$.
For $\alpha \to 0$, we get $S \to \mathbf{1}$; we vary $\alpha$ from zero to one to generate spectra of varying distances from uniformity.
For each $\alpha$, we sample 100 left-singular-value matrices $U$ from the uniform distribution over orthonormal $U$ with column means equal to 0; see \cref{app:zeroMean} for how to generate such matrices.
We fix a unit-norm $\theta^* \in \R^D$ and for each $U$, we generate data from a well-specified linear model, $y_n = \iprod{x_n}{\theta^*} + \eps_n$, where the $\eps_n$ are drawn i.i.d.\ from $\mathcal{N}(0,\sigma^2)$ with variance $\sigma^2=0.5$.
In particular, for each setting of $U$, we generate 100 vectors $Y$. For each setting of $U$ and $Y$, we compute $\L$ and check whether it is quasiconvex.
In the top left panel of \cref{fig:deltaSizeAndNuAssum}, we report the fraction of problems (out of the $100 * 100 = 10{,}000$ datasets for the corresponding $\alpha$ value) with a non-quasiconvex $\L$ versus the distance from uniformity, $\| S - \mathbf{1} \|_1$.
We see that, even for $N = 20$, the fraction of non-quasiconvex problems quickly hits zero as $\| S - \mathbf{1} \|_1$ shrinks.

\textbf{Importance of \cref{assum:ordernu}.} We now establish the necessity of \cref{assum:ordernu} on the decay of $\numax$ with $N$.
To do so, we generate two sets of matrices $U$ as $N$ grows. We generate the first set to satisfy \cref{assum:ordernu}, and we generate the second to violate \cref{assum:ordernu}. In both cases, we will take $D=5$ and ten settings of $N$ between $N = 10$ and $N = 300$.

To generate the assumption-satisfying matrices $U$, we proceed as follows.
For each $N$, we draw 500 matrices $U$ from the uniform distribution over orthonormal $U$ matrices with column means equal to 0.
For each $N$, we plot the \emph{maximum} value of $\numax$ across these 500 $U$ matrices in \cref{fig:deltaSizeAndNuAssum} (top-right) as a blue dot. We fit a line to these values on a log-log plot, and find the slope is -0.74. This confirms that these matrices satisfy \cref{assum:ordernu}.

To generate assumption-violating matrices $U$, we proceed as follows.
Recall that the smallest $N$ is 10 and $D=5$. 100 times, we randomly draw a $U_{\mathrm{small}} \in \R^{8 \times 5}$.
We then construct each $U$ by appending $N-8 \times D$ zeros to $U_{\mathrm{small}}$.
For each $N$, we plot the \emph{minimum} value of $\numax$ across these 100 $U$ matrices in \cref{fig:deltaSizeAndNuAssum} (top-right) as an orange dot.
Since the minimum of $\numax$ is constant with $N$, \cref{assum:ordernu} is violated.

Now we check quasiconvexity. To that end, we randomly select a fixed unit-norm vector $\theta^* \in \R^D$.
For each $N$, we generate 100 noise vectors $E \in \R^N$, where the entries $E_n$ are drawn i.i.d.\ from $\mathcal{N}(0, 0.5)$.
For each $U$ and $E$, we construct $Y = US\theta^* + E$, where $S = \mathbf{1} \in \R^D$.
We then compute the fraction of these ($100*100$ = 10{,}000) losses $\L$ that are non-quasiconvex.
In the lower panel of \cref{fig:deltaSizeAndNuAssum}, we plot this fraction against $N$ for both for the assumption-satisfying case (blue) and the assumption-violating case (orange) in \cref{assum:ordernu}.
When the assumption is satisfied (blue), we see that the conclusion of \cref{thm:uniformSpectrum} holds: beyond a certain $N$, there are no settings of $U$ or $Y$ that generate a non-quasiconvex $\L$. We see that in practice, the boundary $N$ is small or moderate (below 50).
When the assumption is violated (orange dots), we see that the conclusion of \cref{thm:uniformSpectrum} fails to hold: as $N$ grows, there are still settings of $U$ and $Y$ for which quasiconvexity fails to hold.
Finally, we call attention to the vertical axis.
Even in the assumption-violating case, the fraction of non-quasiconvex losses is small. It follows that, even for our degenerate $U$'s, nearly every combination of noise and $U$ leads to a quasiconvex $\L$.
An interesting challenge for future work is to provide a precise characterization of this effect.
\begin{figure}
	\centering
	\begin{tabular}{cc}
		\includegraphics[scale=0.4]{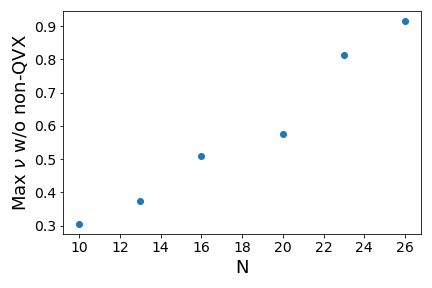}
		\includegraphics[scale=0.4]{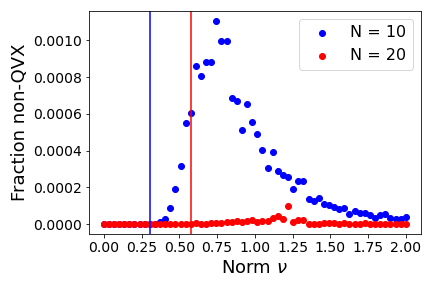}
	\end{tabular}
	\caption{Checking \cref{assum:ONehat}. (\emph{Left}): For each $N$ and each $\nu = \| \hat E\|$, we generate many datasets and check if there is any non-quasiconvex $\L$. We plot the largest $\nu$ for which we find only quasiconvex $\L$.
	The growth is roughly linear, which suggests \cref{assum:ONehat} cannot be loosened.
	(\emph{Right}): For each $N$ and each $\nu = \| \hat E\|$, we generate many data sets and plot the fraction of $\L$ that are non-quasiconvex.
	Vertical lines show $\nu_{max}$ for each $N$.
	}
	\label{fig:epsilonSize}
\end{figure}

\textbf{Do the $\ehat$ need to be small (\cref{assum:ONehat})?} Finally, we demonstrate the necessity of \cref{assum:ONehat}, which can be restated as requiring that $\| \Ehat \|^2$ grows at most linearly in $N$. But we also find a suggestion that there may be even more permissive assumptions of interest.
To this end, we vary $N$ from 10 to 30.
For each $N$, we generate 4{,}000 settings of $U$, each uniform over orthonormal $U$ with column means equal to 0.
For each $U$, we generate 250 unit vectors $R$ such that $U^T R= 0$ (each generated uniformly over such $R$ -- see \cref{app:zeroMean}).
Separately, we consider 60 different norms $\nu$ for the vector $\hat E$ equally spaced between $\nu = 0$ and $\nu = 2$; these are the same across $N$. We generate a single unit-norm $\theta^* \in \R^D$.

For each setting of $U$, $R$, and $\nu$, we consider the regression problem with covariate matrix $U\mathbf{1}$ and responses $Y = U\mathbf{1} \theta^* + \hat E$, where $\hat E := \nu R$. We record whether $\L$ is quasiconvex or not for this problem.
For a particular $N$ and particular error-norm $\nu$, we check whether \emph{any} of the $\L$ (across $4{,}000 * 250 = 1{,}000{,}000$ problems) were non-quasiconvex.
Finally, for each $N$, we find the maximum error-norm $\nu_{\max, N}$ such that for all $\nu < \nu_{\max, N}$ every regression problem is quasiconvex. We plot a dot at $(N, \nu_{\max,N})$ in the left panel of \cref{fig:epsilonSize}. We see that in fact the boundary of allowable $\hat E$ norms does grow about linearly in $N$.
Our next plot lends additional insight into how the boundary $\nu_{\max, N}$ varies with $N$ and is also suggestive of other potential variations on \cref{assum:ONehat} that might be of interest. In particular, in the right panel of \cref{fig:epsilonSize}, we consider two particular values of $N$: $N=10$ (blue) and $N=20$ (red). For each setting of $\nu$ on the horizontal axis, we compute the fraction of non-quasiconvex losses $\L$ over all settings of $U$ and $R$ ($4{,}000 * 250 = 1{,}000{,}000$ problems for each $\nu$).

We see that, as expected from the left panel of \cref{fig:epsilonSize}, the boundary $\nu_{\max,N}$ is higher for $N=20$ than for $N=10$. 
Surprisingly, we also see that at high values of $\nu$, the fraction of non-quasiconvex cases decreases again. We conjecture that this trend is more general (beyond these two particular $N$).
Finally, we note that, as in the bottom panel of \cref{fig:deltaSizeAndNuAssum}, the fraction of non-quasiconvex cases across all $\nu$ is low. Again, this small fraction suggests a direction for future work.

\section{Discussion}
\label{sec:discussion}
We have shown that the LOOCV loss $\L$ for ridge regression can be non-quasiconvex in real-data problems. Local optima need not be global optima. These multiple local optima may pose a practical problem for common hyperparameter tuning methods like gradient-based optimizers, which may get stuck in a local optimum, and grid search, for which upper and lower bounds need to be set.

We proved that the quasiconvexity of $\L$ is determined by only a few aspects of a linear regression problem. But we also showed that the quasiconvexity of $\L$ is still a complicated function of the remaining quantities, and as of this writing the nature of this function is far from fully understood. 
Nonetheless, we have provided theory that guarantees at least some useful cases when $\L$ is quasiconvex: when the spectrum of the covariate matrix is sufficiently flat, the least-squares fit $\hat\theta$ fits the data reasonably well, and the left singular vectors of the covariate matrix are regular. In our experiments, we have confirmed that these assumptions are necessary to some extent: when they are not satisfied, $\L$ can be non-quasiconvex.
Still, our empirical results make it clear there is more to be explored.
We describe some of the directions we believe are most interesting for future work below.

\textbf{Sharper characterization of when $\L$ is quasiconvex.} \cref{fig:qvxComplicated} shows that non-quasiconvexity (yellow) disappears as the spectrum of $X$ becomes uniform; however, it is clear that there is very regular behavior to the pattern of quasiconvexity even when the singular values of $X$ are non-uniform.
We are not able to characterize these patterns at this time but believe these patterns pose a fascinating challenge for future work.
Relatedly, our experiments (\cref{sec:uniformSpectrumExperiments}) show that when our assumptions are violated, quasiconvexity of $\L$ is not guaranteed.
However, we have observed that even when $\L$ is not guaranteed to be quasiconvex, many settings of $U$ and $Y$ still give quasiconvexity. In many of our experiments, the fraction of non-quasiconvex losses $\L$ was extremely small.

\textbf{How many local optima and how bad are they?} Without the guarantee of a single, global optimum, it is not clear that we can ever know that we have globally (near-)optimized $\L$.
However, notice that our examples in \cref{fig:nonQVXExample} all have at most two local optima.
In simulated experiments, we also typically encountered two local optima in non-quasiconvex losses, although we have not studied this behavior systematically.
If $\L$ were guaranteed to have only two or some small number of optima, optimization might again be straightforward, even in the case of non-quasiconvexity; an algorithm could search until it finds the requisite number of optima and then report the one with the smallest value of $\L$.
Alternatively, one might hope that all local optima have CV loss (and ideally out-of-sample error) close in value to that of the global optimum. Indeed, this closeness has been recently argued for certain losses in deep learning \citep{kawaguchi:2016:deepLearningLocalMinima}.
Presumably it is not universally the case that local optima exhibit similar loss since the right panel of \cref{fig:nonQVXExample} seems to give a counterexample. But it might be widely true, or true under mild conditions.
Meanwhile, in the absence of such guarantees, optimization of $\L$ should proceed with caution.

\textbf{Beyond ridge regression.} We have shown -- in our opinion -- surprising non-quasiconvexity for the LOOCV loss for ridge regression.
Do similar results hold for simple models outside ridge regression?
The regularization parameter in other $\ell_2$ or $\ell_1$-regularized generalized linear models is often tuned by minimizing a cross-validation loss.
In preliminary experiments, we have found non-quasiconvexity in $\ell_2$-regularized logistic regression.
To what extent do empirical results like those in \cref{fig:qvxComplicated} or theoretical results like those in \cref{thm:uniformSpectrum} hold for other models and regularizers?

\ifarxiv
\subsubsection*{Acknowledgements}
WS and TB thank an NSF Career Award and an ONR Early Career Grant for support. 
MU and ZF gratefully acknowledge support from NSF Award IIS-1943131, the ONR Young Investigator Program, and the Alfred P. Sloan Foundation.	
\fi

\bibliography{references}
\bibliographystyle{plainnat}

\newpage
\appendix

\section{Real dataset descriptions}
\label{app:realData}

Here, we give the details for the real datasets used to generate \cref{fig:nonQVXExample}.

\paragraph{Life expectancy.} Our first real dataset contains $N = 2{,}938$ observations of life expectancy, along with $D = 20$ covariates such as country of origin or alcohol use. The dataset is available from \citep{rajarshi:2021:lifeExpectancy}.
In this case, $\L$ for the full dataset is quasiconvex. 
But now consider some standard data pre-processing. Practitioners often perform principal component regression (PCR) with the aim of reducing noise in the estimated $\theta$. That is, they take the singular value decomposition of $X = USV$; they then produce an $N \times R$ dimensional covariate matrix $X'$ by retaining just the upper $R$ singular values of $X$: $X' = U_{\cdot,:R} S_{:R}$. If we include this pre-processing step, the resulting LOOCV curve $\L$ is non-quasiconvex for many values of $R$; in the center panel of \cref{fig:nonQVXExample} we show one example for $R = 15$.

This dataset does contain information about people.
However, it is only reported at the aggregated level by a given country per year.
It is not clear to us whether or not consent was obtained by the individuals living in these countries; however, we feel the publication of such data is unlikely to negatively affect any given individual.
Additionally, while we do not know if the data reveals any identifying information about an individual, we feel it is unlikely to do so, as it is published at the country level.

\paragraph{Wine dataset.} Our second dataset consists of recorded wine quality of $N = 1{,}599$ red wines.
The goal is to predict wine quality from $D = 11$ observed covariates relating to the chemical properties of each wine \citep{cortez:2009:wineDataset,cortez:2009:winePaper}.
We find that subsets of this dataset often exhibit non-quasiconvex $\L$.
We search over 400 random subsets of this dataset of size $N = 50$.
In \cref{fig:nonQVXExample}.
Twelve of these led to non-quasiconvex losses $\L$, and \cref{fig:nonQVXExample} shows one of these examples.

This dataset does not contain information about people, and so concerns about consent and personally identifying information do not seem relevant here.
\section{Proof of \cref{prop:whatMatters}}
\label{app:whatMatters}
We now restate and then prove \cref{prop:whatMatters}.
\begin{repprop}{prop:whatMatters}
	Assume \cref{cond:dataProcessing} holds.
	The quasiconvexity of $\L$ is independent of the following
	\begin{enumerate}
		\item The matrix of right singular vectors, $V$
		\item The norm of the responses, $\n{Y}_2$
		\item The scaling of the singular values (i.e.\ changing $S$ into $S / c$ for $c \in \R_{> 0}$)
	\end{enumerate}
	in the sense that altering any of these quantities does not change whether or not $\L$ is quasiconvex.
\end{repprop}
\begin{proof}
First, it's easiest to write our function of interest in a simpler form:
\begin{equation}
	\L (\lambda) = \sum_{n=1}^N \frac{1}{(1-Q_n(\lambda))^2} (x_n^T \hat\theta_\lambda - y_n)^2,
\end{equation}
where $Q_n(\lambda) := x_n^T \big( X^T X + \lambda I_D\big)\inv x_n$ and $\hat\theta_\lambda := \big( X^T X + \lambda I_D)\inv X^T Y$.
\\\\
Let the singular value decomposition of $X$ be $X = USV$. Then:

\paragraph{$\mathbf{V}$ does not affect does not affect the quasiconvexity of $\L$.} To prove this claim, note that $x_n = u_n^T SV$, where $u_n$ is the $n$th row of $U$. So:
$$
	Q_n(\lambda) = u_n^T S V^T \big( V (S^2 + \lambda I_D)\inv V^T \big) V S u_n = u_n^T S(S^2 + \lambda I_D )\inv Su_n.
$$
So $Q_n(\lambda)$ is actually independent of $V$. Next,
\begin{align*}
	x_n^T \hat\theta_\lambda &= u_n^T SV^T V (S^2 + \lambda I_D)\inv V^T V S U^T Y \\
	&= u_n^T S (S^2 + \lambda I_D)\inv S U^T Y,
\end{align*}
which is also independent of $V$.

\paragraph{$\mathbf{\| Y\|_2^2}$ does not affect the quasiconvexity of $\L$.} In particular, we can treat $Y$ as sitting on the $D$-dimensional unit sphere. 
To see this, take two different $Y$'s related by a scaling: $y_n^{(1)} = c y_n^{(0)}$ for some scalar $c \in \R$. Then, using the same superscripts:
\begin{align*}
	& \hat\theta_\lambda^{(1)} = \big( X^T X + \lambda I_D \big)\inv X^T Y^{(1)} = c\hat\theta_\lambda^{(0)}.
\end{align*}
So, we can relate the two LOOCV functions by:
\begin{equation}
	\L^{(1)}(\lambda) = \sum_n \frac{1}{(1-Q_n(\lambda))^2} \big(c x_n^T \hat\theta_\lambda^{(0)} - cy_n^{(0)} \big)^2 = c^2 \L^{(0)}(\lambda).
\end{equation}
So scaling $Y$ by $c$ uniformly scales $\L(\lambda)$ by $c^2$.
Mutliplying $\L$ by a constant does not affect is quasiconvexity.

\paragraph{The scaling of the singular values $s_1, \dots, s_D$ does not affect the quasiconvexity of $\L$.} 
We have already shown that $V$ does not affect the quasiconvexity of $\L$, so fix $V = I_D$ to simplify the proof.
Pick some scaling $c > 0$, and fix some spectrum $S^{(1)}$. 
Define $S^{(0)} := cS^{(1)}$.
Using the same superscripts, we have:
\begin{align}
	Q^{(0)}_n(\lambda) &= \sum_{d=1}^D u_{nd}^2 \frac{c^2s_d^2}{c^2 s_d^2 + \lambda} = c^2 \sum_{d=1}^D u_{nd}^2 \frac{s_d^2}{ s_d^2 + (\lambda / c^2)} = Q^{(1)}\left( \frac{\lambda}{c^2} \right)
	\label{QnScaling}
\end{align}
Similarly, define $(x_n^T \hat\theta)^{(0)}(\lambda)$ to be the inner product of $x_n^{(0)}$ and $\hat\theta^{(0)}(\lambda)$. Then:
\begin{align}
	(x_n^T \hat\theta)^{(0)}(\lambda) &= u_n^T S^{(0)} \bigg( (S^{(0)})^2 + \lambda I_D \bigg) U^T Y \\
	&= u_n^T S^{(1)} \bigg( (S^{(1)})^2 + \frac{\lambda}{c^2} I_D \bigg)\inv U^T Y \\
	&= (x_n^T \hat\theta)^{(1)}\left( \frac{\lambda}{c^2} \right).
\end{align}
This, along with \cref{QnScaling} implies that $\L^{(0)}(\lambda) = \L^{(1)}(\lambda / c^2)$.
That is, multiplying the singular values by $c$ stretches out $\L$ by a factor of $c$.
This does not change the quasiconvexity of $\L$.
\end{proof}
\section{Empirical validation of \cref{assum:ordernu}}
\label{app:orderNu}

\begin{figure}
	\centering
	\includegraphics[scale=0.6]{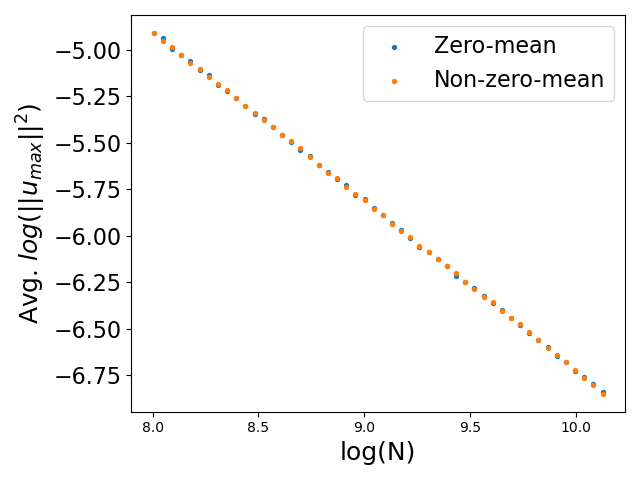}
	\caption{Experiment from \cref{app:orderNu}.
	Orange dots show the decay of $\numax$ for uniformly drawn non-zero-column-mean orthonormal matrices $U$; we see that, as proven by \citet{candes:2009:matrixCompletion}, these matrices satisfy \cref{assum:ordernu}.
	Blue dots show the decay of $\numax$ for uniformly drawn zero-column-mean orthonormal matrices.
	While such matrices are not known to satisfy \cref{assum:ordernu}, we see that their $\numax$ decays at exactly the same rate as the non-zero-column-mean matrices.
	}
	\label{fig:empiricalOrderNu}
\end{figure}

As noted in the main text, \cref{assum:ordernu} can be interpreted as an assumpution about the coherence of the $U$ matrix, a quantity commonly found in the compressed sensing literature \citep{candes:2009:matrixCompletion}.
In particular, \cref{assum:ordernu} requires that the coherence of $U$ decay with $N$ sufficiently fast. 
Similar conditions have been studied in the literature for other matrices.
For example, Lemma 2.2 of \citet{candes:2009:matrixCompletion} shows that if $U$ is drawn uniformly at random from the set of all orthonormal $N \times D$ matrices, then $\max_n \numax = O(\log(N) / N)$ with probability going to 1 as $N \to \infty$.
As $\log(N) / N$ tends towards zero faster than $N^{-p}$ for any $0 < p < 1$, Lemma 2.2 of \citet{candes:2009:matrixCompletion} proves that \cref{assum:ordernu} holds with high probability if $U$ is drawn uniformly from the set of orthogonal matrices.

However, the $U$'s of interest here have an additional constraint: that their columns be zero-mean.
It is not clear how to adopt the proof of \citet{candes:2009:matrixCompletion} to this situation.
Instead, we offer empirical evidence that \cref{assum:ordernu} holds in the case that $U$ is drawn uniformly at random from the set of zero-column-mean matrices.
We describe how to generate such matrices in \cref{app:zeroMean}.
For fifty values of $N$ from $N = 2{,}500$ to $N = 20{,}500$, we draw 750 zero-mean orthonormal matrices $U$ from the uniform distribution.
We plot the average $\numax$ over these 750 replicas on a log scale versus $N$ in \cref{fig:empiricalOrderNu} (orange dots).
For comparison, we plot the average $\numax$ over 750 replicas when $U$ is drawn uniformly from the set of all orthonormal matrices (no zero-mean constraint) as the blue dots.
The decay of $\numax$ with and without the zero-mean constraint is essentially identical.
Given this experiment, we argue that, although theoretically unjustified, \cref{assum:ordernu} places only modest restrictions on the regression problems to which \cref{thm:uniformSpectrum} applies.

\section{Proof of \cref{thm:uniformSpectrum}}
\label{app:proofs}

\label{app:uniformSpectrum}

We begin with some technical lemmas.

\begin{lm} \label{lm:app-perturbSumZero}
	Take real numbers $s_1, \dots, s_N$ and $r_1, \dots, r_N$, where $r_n \in [\ell, u]$ and $\sum_{n=1}^N s_n = 0$. Then:
	$$
		\left| \sum_{n=1}^N r_n s_n \right| 
		  \leq 
		\frac{u-\ell}{2} \sum_{n=1}^N \abs{s_n}.		                         
	$$
\end{lm}
\begin{proof}
	Let $s_n^+ := \max(0, s_n)$ and $s_n^- := \max(0, -s_n)$. Then the condition $\sum_n s_n = 0$ implies that $\sum_n s_n^+ = \sum_n s_n^- = (1/2) \sum_n \abs{s_n}$. So:
	$$
	\left| \sum_{n=1}^N r_n s_n \right|  \leq u\sum_{n=1}^N s_n^+ - \ell \sum_{n=1}^N s_n^- = \frac{u-\ell}{2} \sum_{n=1}^N \abs{s_n}.
	$$
\end{proof}
Now we state a useful consequence of our above assumptions:
\begin{prop}
	Take \cref{assum:orderThetaHat,assum:ordernu,assum:ONehat}. We have:
	$$
		\sum_{n=1}^N (1-\nun) \ehat \uth = o(1)
	$$
	\label{prop:app-crossTerm}
\end{prop}
\begin{proof}
	Notice that $\sum_n \ehat \uth = E^T U\hat\theta = 0$. So, we are trying to bound $ \left| \sum_n \nun \ehat \uth \right|$. \cref{assum:ordernu} implies that $\nun$ is $O(N^{-p})$. As $\nun$ is lower bounded by zero, we can apply \cref{lm:app-perturbSumZero} to get the upper bound:
	\begin{equation}
		\left| \sum_{n=1}^N \nun \ehat \uth \right|
		   \leq
		O(N^{-p}) \sum_{n=1}^N \abs{ \ehat \uth}.
	\end{equation}
	By Cauchy-Schwarz, we can upper bound the sum as:
	\begin{equation}
	 	\sum_n \abs{\ehat \uth} \leq \left( \sum_{n=1}^N \ehat^2 \right)^{1/2} \left( \sum_{n=1}^N \uth^2 \right)^{1/2}
	 \end{equation}
	 By \cref{assum:ONehat} and the fact that $\sum_n \uth^2 = \| \hat\theta \|^2$, we have overall:
	 \begin{equation}
	    \left| \sum_{n=1}^N \nun \ehat \uth \right| = O(N^{-p}) O(\sqrt{N}) \n{\hat\theta}^2 = o(1),
	 \end{equation}
	 where the final equality holds because by assumption, $p > 1/2$ and $\| \hat\theta\| = O(1)$
\end{proof}

We restate and prove \cref{thm:uniformSpectrum} from \cref{sec:uniformSpectrum}
\begin{reptheorem}{thm:uniformSpectrum}
	Take \cref{assum:ONehat,assum:orderThetaHat,assum:ordernu,assum:positiveCoeffs}.
	For $N$ sufficiently large, there is a neighborhood $\Delta$ of $\mathbf{1} \in \R^D$ such that if the spectrum $S \in \Delta$, $\L(\lambda)$ is quasiconvex.
\end{reptheorem}
\begin{proof}
We first prove the theorem for an exactly uniform spectrum, $S = \mathbf{1}$. 
To do so, we work with a sufficient condition for a one-dimensional function $\L$ to be quasiconvex: for all $\lambda$ such that $\L'(\lambda) = 0$, we have $\L''(\lambda) > 0$ \citep[Chapter 3.4]{boyd:2009:cvxOpt}.
With this characterization of quasiconvexity in mind, our proof can be broken into two steps.
We sketch each step here and refer to later lemmas for their proofs.
\begin{enumerate}
	\item \textbf{Bound the region where $\L'$ can be zero.} Write $\L'$ as:
	$$
		\L'(\lambda) = \frac{1}{(\lambda + 1)^4} \sum_{n=1}^N \frac{(\lambda + 1)^3}{(\lambda + 1 - \nun)^3} \left( \xi_{n1} \lambda^2 + \xi_{n2}\lambda + \xi_{n3} \right)
	$$
	To find where this can be zero, we can ignore the $1/(\lambda + 1)^4$. Then this is \emph{almost} a quadratic in $\lambda$. Find the most positive root of the quadratic, $\lambda_Q = \Omega(1)$. Bound this function's perturbations away from a quadratic, and bound how much this can increase $\lambda_Q$. \textbf{Result:} $\L'(\lambda)$ can only be zero for $\lambda \in [0, \lambda_Q + o(1)]$. We prove this step in \cref{lm:app-LprimeRootBound} below.
	\item \textbf{Show $\L''(\lambda) > 0$  for any $\lambda \in [0,\lambda_Q+o(1)]$ for which $\L'(\lambda) = 0$.} Essentially the same strategy; write
	$$
	\L'(\lambda) = \frac{1}{(\lambda + 1)^5} \sum_{n=1}^N \frac{(\lambda + 1)^4}{(\lambda + 1 - \nun)^4} \left( a_n \lambda^2 + b_n \lambda + c_n \right)
	$$
	This is a roughly a bowl-down quadratic with only one root bigger than zero; i.e. it is positive for $\lambda \in [0,\lambda_Q']$, where $\lambda_Q' > \lambda_Q + O(1)$ is the location of the quadratic's rightmost root. Show that the  deviations away from quadratic imply that $\L''$ is positive for $\lambda \in [0, \lambda_Q' - o(1)] = [0, \lambda_Q + O(1) - o(1)]$. We prove this step in \cref{lm:app-secondDerivPositive} below.
\end{enumerate}
With the theorem proved for an exactly uniform spectrum, we note that $\L'$ and $\L''$ are continuous functions of the spectrum $S$. As $\L''$ is stricly bounded away from zero on a region that contains $[0,\lambda_Q + O(1)]$, by continuity in the singular values, there is a neighborhood $\Delta$ of $\mathbf{1} \in \R^D$ such that if $S \in \Delta$, $\L''(\lambda) > 0$ for all $\lambda$ for which $\L'(\lambda) = 0$.
\end{proof}
Before proving the lemmas needed for  \cref{thm:uniformSpectrum}, we first prove a technical lemma.
\begin{lm} \label{lm:app-deltaBound}
	Take \cref{assum:ONehat,assum:ordernu}. We have:
	\begin{align}
		-\sum_{n=1}^N \left( \frac{1}{(1-\nun)^3} - 1 \right) \nun \ehat^2 = o(1).
	\end{align}
\end{lm}
\begin{proof}
	First, note that as $0 \leq \nun < 1$, this quantity is strictly negative. 
	So, it suffices to lower bound it by a quantity that is $o(1)$.
	We apply the lower bound
	$$
		-\sum_{n=1}^N \left( \frac{1}{(1-\nun)^3} - 1 \right) \nun \ehat^2
		\geq
		-\left(\frac{1}{(1-\numax)^3} - 1 \right) \numax \sum_{n=1}^N \ehat^2.
	$$
	By a Taylor expansion around $\numax = 0$, we have:
	$$
		= - \bigg( \numax + 6(\numax)^2 + O((\numax)^3) - \numax \bigg) \sum_{n=1}^N \ehat^2.
	$$
	By \cref{assum:ONehat,assum:ordernu}, this is equal to $O(N^{-2p}) O(N) = o(1)$.
\end{proof}
Before getting into the proofs of our main lemmas, we can first rearrange $\L'$ into a convenient form.
First, we can do some algebra to show that when $S = \mathbf{1} \in \R^D$, we have:
\begin{align}
	\L'(\lambda) = \sum_{n=1}^N & \frac{2}{(1 - \frac{1}{1+\lambda }\nun)^2}  \left(\frac{1}{1+\lambda} u_n^T U^T Y - y_n \right) \\
& * \left[ -\frac{1}{1 - \frac{1}{1+\lambda}\nun} \frac{1}{1+\lambda}\nun \left( \frac{1}{1+\lambda} u_n^T U^T Y - y_n\right) - \frac{1}{(1+\lambda)^2} u_n^T U^T Y \right] \nonumber
\end{align}
Decomposing $Y = X\hat\theta + \hat E$, where $\hat E := Y - X\hat\theta$, we have $U^T \hat E = 0$. After doing some algebra, we obtain:
\begin{align}
	\L'(\lambda) = \frac{2}{(1+\lambda)^4} \sum_{n=1}^N \frac{(1+\lambda)^3}{(1+\lambda-\nun)^3} & \left( (1-\nun) \uth^2 - \nun \ehat^2 + (1-2\nun) \ehat \uth \right) \lambda^2 \nonumber \\
	& + \left( (1-\nun) \uth^2 - 2\nun \ehat^2 + (2-3\nun) \ehat \uth \right) \lambda \nonumber \\
	& + \left( -\nun \ehat^2 + (1-\nun) \ehat \uth \right) \label{quadraticLprime}
\end{align}
Now we can prove the Lemmas needed to prove \cref{thm:uniformSpectrum}.
\begin{lm} \label{lm:app-LprimeRootBound}
	For a flat spectrum $S = \mathbf{1} \in \R^D$, there is some $\lambda_Q$ that is $\Theta(1)$ such that $\L'(\lambda) = 0$ implies that $\lambda \in [0,\lambda_Q + o(1)]$.
\end{lm}
\begin{proof}
	First, we can discard the $1/(1+\lambda)^4$ in front of \cref{quadraticLprime} for the purposes of deciding where $\L' = 0$; let $g(\lambda) = \L'(\lambda) (1+\lambda)^4$:
	\begin{equation} \label{quadraticWithXis}
		g(\lambda) = \sum_{n=1}^N \frac{(1+\lambda)^3}{(1+\lambda - \nun)^3} \left( \xi_{n1} \lambda^2 + \xi_{n2} \lambda + \xi_{n3} \right),
	\end{equation}	
	where $\xi_{n1},\xi_{n2},\xi_{n3} \in \R$ are the appropriate quantites read off from \cref{quadraticLprime}.
	Notice that $g$ is nearly a quadratic; in particular, if $\nun = 0$, then $g$ is a quadratic. The idea is to let $\lambda_Q$ be the rightmost root of the quadratic if $\nun = 0$ for all $n$; we then show that the perturbations away from this quadratic are small enough to imply that all zeros of $g$ lie in $[0, \lambda_Q + o(1)]$.
	
Write $\xi_{\cdot i} := \sum_n \xi_{ni}$. Then, via the quadratic formula, the roots of $g_Q(\lambda) := \xi_{\cdot 1} \lambda^2 + \xi_{\cdot 2} \lambda + \xi_{\cdot 3}$ are
$$
	\lambda = \frac{-\xi_{\cdot 2} \pm \left[ \xi_{\cdot 2}^2 - 4\xi_{\cdot 1} \xi_{\cdot 3}\right]^{1/2}}{2 \xi_{\cdot 1}}.
$$
By \cref{prop:app-crossTerm,assum:positiveCoeffs}, the positive root of this quadratic is positive and $\Theta(1)$. Let this positive root be $\lambda_Q$. Now we need to bound the devitions of $g$ away from the quadratic $g_Q$. Let $\delta(\lambda) := g(\lambda) - g_Q(\lambda)$ be these deviations:
\begin{equation}
	\delta(\lambda) := \sum_{n=1}^N \left( \left( \frac{\lambda+1}{\lambda+1-\nun}\right)^3 - 1 \right) \big( \xi_{n1}\lambda^2 + \xi_{n2}\lambda + \xi_{n3} \big).
\end{equation}
Notice that our quadratic $g_Q$ is convex ($\xi_{\cdot 1} > 0)$ by \cref{assum:positiveCoeffs}. Thus the way to move the roots of $g$ further right than $\lambda_Q$ is to have $\delta(\lambda)$ be negative. 
We can lower bound $\delta(\lambda) \geq \delta(0)$ by noting that $\xi_{\cdot, 1}, \xi_{\cdot 2} > 0$ and $\xi_{\cdot, 3} < 0$. Thus:
\begin{align*}
	\delta(\lambda) &\geq \delta(0) = -\sum_{n=1}^N \left(\frac{1}{(1-\nun)^3} - 1 \right) \nun \ehat^2.
\end{align*}
By \cref{lm:app-deltaBound}, we have that $\delta(0) = o(1)$. 

As we know $g(\lambda) = g_Q(\lambda) + \delta(\lambda) \geq g_Q(\lambda) + \delta(0)$, the final step of our proof is to find the right-most $\lambda$ for which $g_Q(\lambda) = -\delta(0)$, as beyond such a $\lambda$, $g > 0$. 
In fact, an upper bound on this $\lambda$ via convexity will suffice. Using convexity with the fact that $g_Q(\lambda_Q) = 0$, we have that beyond $\lambda = \lambda_Q + \delta(0) / g_Q'(\lambda_Q)$, $g_Q(\lambda) \geq \delta(0)$, and thus $g(\lambda) \geq 0$. 
Given $g_Q'(\lambda_Q) = 2\lambda_Q \xi_{\cdot 2} + \xi_{\cdot 3}$
\begin{equation}
	\lambda_Q + \frac{\delta(0)}{2\lambda_Q \xi_{\cdot 2} + \xi_{\cdot 3}} = \lambda_Q + \frac{o(1)}{O(1)} = \lambda_Q + o(1),
\end{equation}
as claimed.
\end{proof}
\begin{lm} \label{lm:app-secondDerivPositive}
	Let $\lambda_Q$ be as defined in \cref{lm:app-LprimeRootBound}. For a flat spectrum $S = \mathbf{1} \in \R^D$, we have that $\L''(\lambda) > 0$ for $\lambda \in [0, \lambda_Q + O(1)]$.
\end{lm}
\begin{proof}
	The strategy is similar to the proof of \cref{lm:app-LprimeRootBound}: we show that $\L''$ is nearly a quadratic, find the root of this quadratic, and then show that the location of this root can only change by $o(1)$ due to the deviations away from quadratic.
	\\\\
First, computing $\L''$ from \cref{quadraticLprime} gives
\begin{align}
	\L''(\lambda) = \frac{2}{(1+\lambda)^5} \sum_{n=1}^N \frac{(1+\lambda)^4}{(1+\lambda - \nun)^4} \bigg( & -\xi_{n1}\lambda^2 \\
	& + \big( 2(1-\nun)\xi_{n1} - 2\xi_{n2} \big) \lambda \\
	& + \big( (1-\nun)\xi_{n2} - 3\xi_{n3} \big) \bigg),
\end{align}
where the $\xi_{ni}$'s are as defined in the proof of \cref{lm:app-LprimeRootBound}. As we are interested in the region where $\L'' > 0$, we can neglect the $1/(1+\lambda)^5$ factor in front; define $h(\lambda) := (1+\lambda)^5 \L''(\lambda)$. Now, define the following
\begin{align}
	& a_n := -\xi_{n1}   \\
	& \quad\quad\quad = 2\nun \ehat^2 - (1-\nun) \uth^2 \nonumber \\
	& b_n := 2(1-\nun)\xi_{n1} - 2\xi_{n2} \\
	& \quad\quad\quad =  \big(2(1-\nun)^2 - (1-\nun)\big) \uth^2 + \big( -2(1-\nun)\nun + 4\nun\big) \ehat^2 \nonumber \\
	& \quad\quad\quad\quad + \big(2(1-\nun)(1-2\nun) - 2(2-3\nun)\big) \ehat\uth \nonumber \\
	& c_n := (1-\nun)\xi_{n2} - 3\xi_{n3} \\
	& \quad\quad\quad = (1-\nun)^2 \uth^2 + \big(-2(1-\nun)\nun + 3\nun \big) \ehat^2 \nonumber \\
	& \quad\quad\quad\quad + \big((1-\nun)(2-3\nun) - 3(1-\nun)\big) \ehat\uth \nonumber \\
	& h_Q(\lambda) := \sum_{n=1}^N a_n \lambda^2 + b_n \lambda + c_n \\
	& \delta^{(2)}(\lambda) := \sum_{n=1}^N \left( \left(\frac{1+\lambda}{1+\lambda-\nun}\right)^4 - 1\right) \bigg( a_n \lambda^2 + b_n \lambda + c_n \bigg).
\end{align}
Note that $h = h_Q + \delta^{(2)}$. Let $a_\cdot = \sum_n a_n$, and likewise for $b_\cdot, c_\cdot$. Application of \cref{prop:app-crossTerm,assum:ordernu,assum:ONehat} gives:
\begin{align}
	b_{\cdot} &= \sum_{n=1}^N \bigg( \uth^2  + 2\nun \ehat^2 \bigg) + o(1) \\
	c_\cdot &= \xi_{\cdot 2} - 3\xi_{\cdot 3} + o(1).
\end{align}
In particular, $b_\cdot > 0$ for large enough $N$, and, noting that $\xi_{\cdot 3} < 0$, we have $c_\cdot > 0$ with $\abs{c_\cdot} > \abs{\xi_{\cdot 3}}$ for large enough $N$.
Now, in general $h_Q$ will have two roots:
$$
	\lambda = \frac{-b_\cdot \pm \left[ b_\cdot^2 - 4a_\cdot c_\cdot\right]^{1/2}}{2a_{\cdot}} = \frac{b_\cdot \mp \left[ b_\cdot^2 + 4\xi_{\cdot 1} c_{\cdot}\right]^{1/2}}{2\xi_{\cdot 1}}.
$$
Now, as $[b_\cdot^2 + 4\xi_{\cdot 1} c_\cdot]^{1/2} > b$, only one of these roots is positive; call this root $\lambda'_Q$. We can compare the expression of this positive root to that of $\lambda_Q$ from \cref{lm:app-LprimeRootBound}. 
Noting that $b_\cdot / (2\xi_{\cdot 1}) > 0 > -\xi_{\cdot 2}$, and $c_{\cdot} > 0 > \xi_{\cdot 3}$, where $c_{\cdot} = O(1)$, we have that the positive root is equal to $\lambda_Q + O(1)$. 
Now we need only lower bound $\delta^{(2)}(\lambda)$ on $[0, \lambda_Q + O(1)]$. Using \cref{lm:app-perturbSumZero} with the fact that $(1+\lambda)^4/(1+\lambda - \nun)^4 - 1 = o(1)$ as $\nun \to 0$ by \cref{assum:ordernu}:
\begin{equation}
	\abs{\delta^{(2)}(\lambda)} \leq o(1) \sum_{n=1}^N \abs{a_n O(1) + b_n O(1) + c_n} = o(1) \sum_{n=1}^N O\left( \frac{1}{N} \right) = o(1).
\end{equation}
Thus $h(\lambda) > 0$ for $\lambda \in [0, \lambda_Q + O(1) - o(1)]$.
\end{proof}

\section{Proof of \cref{cor:subGaussian}}
\label{app:subGaussian}
We first give state a theorem about the concentration of i.i.d.\ sub-Gaussian matrices
\begin{thm}[Theorem 4.6.1 from \citet{vershynin:2017:hdpBook}] \label{thm:subGaussianBounds}
Suppose that the $x_n$ are independent sub-Gaussian isotropic random vectors with maximum sub-Gaussian constant $K$. Then for some cosntant $C > 0$ and any $t \geq 0$, the following holds with probability at least $1 - 2e^{-t^2}$:
\begin{equation}
	\sqrt{N} - CK^2 (\sqrt{D} + t) \leq s_D \leq s_1 \leq \sqrt{N} + CK^2 (\sqrt{D} + t).
\end{equation}
\end{thm}

We now restate and then prove \cref{cor:subGaussian}. Let $s_D \leq \dots \leq s_1$ be the singular values of $X$.
\begin{repcor}{cor:subGaussian}
	Take \cref{assum:ordernu,assum:positiveCoeffs}. Assume we have a well-specified linear model for some $\theta^* \in \R^D$; that is, $y_n = \iprod{x_n}{\theta^*} + \eps_n$, where $\eps_n \overset{i.i.d.}{\sim} \mathcal{N}(0,\sigma^2)$. 
	If $\sigma$ is sufficiently small, $\hat\theta$ is consistent for $\theta^*$, and the entries of the covariate matrix $x_{nd}$ are i.i.d.\ sub-Gaussian random variables, then $\L$ is quasiconvex with probability tending to 1 as $N \to \infty$.
\end{repcor}
\begin{proof}
	The idea is to show that our assumptions imply that \cref{assum:ONehat,assum:orderThetaHat} hold, as well as that the spectrum of $X$ becomes uniform with high probability.
	Our assumption that $\hat\theta$ is consistent for $\theta^*$ immediately implies that $\| \hat\theta \| = O(1)$.
	Next, as discussed after \cref{assum:ONehat}, we can stack the $\eps_n$ into a vector $E \in \R^N$. We then have that $\| E \|^2 \geq \|(I_N - UU^T) E \|^2 = \| \hat E \|^2 = \sum_n \ehat^2$.
	As $\| E \|^2 = O(N)$, we have that $\sum_n \ehat^2 = O(N)$, which implies that \cref{assum:ONehat} holds.
	
	Now, by \cref{thm:subGaussianBounds} with $t = N^{1/3}$ (any $t$ that goes to infinity with $N$ but is $o(\sqrt{N})$ will work), we have that the singular values of $X$ satisfy
	\begin{equation}
		\sqrt{N} - o(\sqrt{N}) \leq s_D \leq s_1 \leq \sqrt{N} + o(\sqrt{N}).
	\end{equation}
	with probability at least $1 - o(1)$.
	By \cref{prop:whatMatters}, the quasiconvexity of $\L$ is invariant to a scaling of the singular values; thus we can divide all singular values by $\sqrt{N}$ to get that all singular values lie in the interval $[1-o(1), 1+o(1)]$ with probability going to 1.
	Thus, for any neighborhood $\Delta$ of $\mathbf{1} \in \R^D$, the (normalized) singular values will eventually lie within $\Delta$ with arbitrariliy high probability.
	Thus all of the conditions of \cref{thm:uniformSpectrum} are met, implying that $\L$ will be quasiconvex with probability tending towards 1 as $N \to \infty$.
	
\end{proof}
\section{Generating zero-mean orthonormal matrices uniformly at random}
\label{app:zeroMean}

In our experiments in \cref{sec:uniformSpectrumExperiments}, we draw $N \times D$ orthonormal matrices $U$ with zero-mean columns from the uniform distribution over such matrices.
To do so, we generate vectors $a_1, \dots, a_D \in \R^D$ such that $a_{nd} \overset{i.i.d.}{\sim} \mathcal{N}(0,1)$.
We then use the Gram-Schmidt process to orthogonalize the vectors $\{\mathbf{1}, a_1, \dots, a_D\}$; the second through $D+1$th outputted vectors make up the columns of $U$.
Notice this procedure requires $N < D$.

In our experiments surrounding \cref{assum:ONehat}, we need to generate vectors $R$ such that $U^T R = 0$ uniformly over such $R$'s.
To do so, we generate a vector $a \sim \mathcal{N}(0, I_N)$.
We then compute $b = (I_N - UU^T)a$; setting $R = b / \| b \|$ yields the result.

Why is this uniform over all vectors in the null space of $U$? 
Recall that $a$ is an isotropic random vector.
It is well-known that this implies that $a / \| a \|$ is uniform over the unit sphere.
As multiplication by $I_N - UU^T$ is an orthogonal projection, we have that $b = (I_N - UU^T)a$ is isotropic over the null space of $U$.
Thus $b / \| b \|$ is uniform over the null-space of $U$.
The same reasoning shows that the use of the Gram-Schmidt algorithm to generate orthonormal zero-column-mean matrices $U$ is uniform over such matrices -- we start with isotropic random vectors, Gram Schmidt applies orthogonal projections to each and then normalizes the results.
\section{Replicating experiments with error bars}
\label{app:errorBars}
\begin{figure}
	\centering
	\begin{tabular}{cc}
		\includegraphics[scale=0.43]{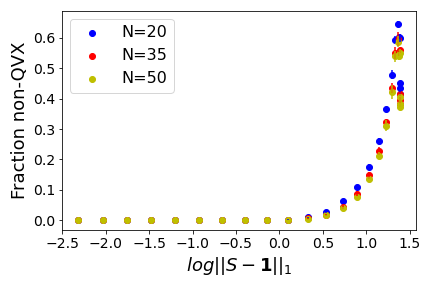} &
		\includegraphics[scale=0.43]{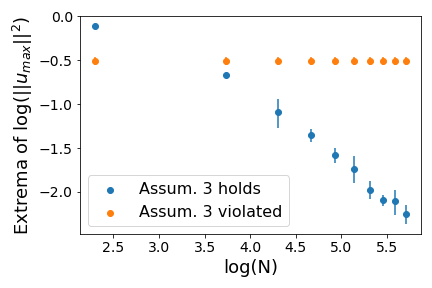}
	\end{tabular}
	\centering
	\includegraphics[scale=0.43]{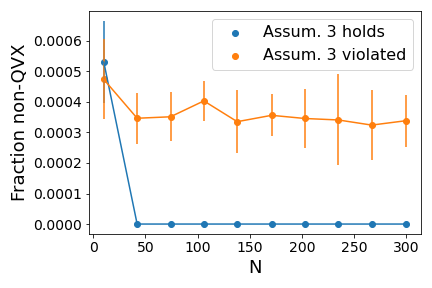}
	\caption{
	Replication of \cref{fig:deltaSizeAndNuAssum} with error bars; we repeat the original caption here for reading convenience.
	(\emph{Upper left}):
	We generate many datasets and plot the fraction that are not quasiconvex,
	varying $N$ and the distance of the spectrum from uniformity ($\|S - \mathbf{1}\|_1$).
(\emph{Upper right}): We generate two sets (orange, blue) of left-singular vector matrices $U$. In the blue case, we check that the maximum of $\log \numax$ across all $U$ for a particular $N$ decreases roughly linearly on a log-log plot (i.e.\ the blue set satisfies \cref{assum:ordernu}). In the orange case, we check that the minimum of $\log \numax$ across all $U$ for a particular $N$ is roughly constant (i.e.\ the orange set does not satisfy \cref{assum:ordernu}). (\emph{Lower}): For all the $U$ matrices from the upper right plot,
we generate many datasets and plot the fraction that are not quasiconvex.}
\label{fig:app-errorBars}
\end{figure}
For each of the experiments in \cref{fig:deltaSizeAndNuAssum} (our experiments about the size of the neighborhood $\Delta$ and $U$'s violating \cref{assum:ordernu}), we repeat the experiment five times to understand the random variability in each experiment.
\cref{fig:app-errorBars} shows the result. 
All plots are created exactly as in \cref{fig:deltaSizeAndNuAssum}, except each dot is now an average over all five trials.
Error bars are equal to two times the standard deviation across these five trials.
We see that our conclusions from \cref{fig:deltaSizeAndNuAssum} still hold.
In fact, many of the error bars are so small that they are barely visible on the scale of these plots.

\end{document}